%% file: root.tex
\documentclass[letterpaper, 10 pt, conference]{ieeeconf}  

\IEEEoverridecommandlockouts                              

\input{preamble}

\makeatother

\DeclareCaptionFont{mysize}{\fontsize{8}{9.6}\selectfont}
\title{\LARGE \bf OcclusionCBF: Backup Control Barrier Functions for Safe Navigation Among Hidden Dynamic Obstacles}
\author{Taekyung Kim$^{1, *}$, Hun Kuk Park$^{1, *}$, Renya Wada$^{2}$, Nikolay Atanasov$^{3}$, Shumon Koga$^{2}$, Dimitra Panagou$^{1,4}$
\thanks{$^{*}$These authors contributed equally to this work}
\thanks{$^{1}$Department of Robotics, $^{4}$Department of Aerospace Engineering, University of Michigan, Ann Arbor, MI 48109, USA {\tt\footnotesize \{taekyung, parkcart, dpanagou\}@umich.edu} } 
\thanks{$^{2}$Department of Computer Science and Systems Engineering, Kobe University, Hyogo, Japan {\tt\footnotesize 267x093x@stu.kobe-u.ac.jp, koga@harbor.kobeu.ac.jp  } } 
\thanks{$^{3}$Department of Electrical and Computer Engineering, University of California, San Diego, CA 92093, USA {\tt\footnotesize natanasov@ucsd.edu} }
}
\begin{document}
\maketitle
\thispagestyle{empty}
\pagestyle{empty}


\input{content}

\addtolength{\textheight}{0 cm}   





\bibliographystyle{IEEEtran}
\typeout{}
\bibliography{references.bib}

\end{document}

%% file: preamble.tex
\usepackage{graphics} 
\usepackage{subfig} 
\usepackage{wrapfig}

\usepackage{amsthm} 
\usepackage{amsmath,amssymb}
\usepackage{graphicx}
\usepackage{tabularx}
\usepackage{multirow, makecell}
\usepackage{diagbox}
\usepackage{slashbox}
\usepackage{rotating}
\usepackage{cite}

\usepackage{url}
\usepackage{bm}
\usepackage[table]{xcolor}
\usepackage{hyperref}
\usepackage{siunitx} 
\usepackage{booktabs}
\usepackage{cleveref}
\usepackage{mathtools} 
\usepackage{upgreek} 

\let\labelindent\relax 
\usepackage{enumitem} 

\usepackage{lipsum}

\usepackage[ruled,vlined]{algorithm2e}

\usepackage[nolist, nohyperlinks]{acronym}
\acrodef{MPC}[MPC]{Model Predictive Control}
\acrodef{QP}[QP]{Quadratic Program}
\acrodef{CBF}[CBF]{Control Barrier Function}

\newcommand{\vx}{{\boldsymbol x}}
\newcommand{\vu}{{\boldsymbol u}}
\newcommand{\vq}{{\boldsymbol q}}

\newcommand{\vy}{{\boldsymbol y}}

\newcommand{\vp}{{\boldsymbol p}} 
\newcommand{\vv}{{\boldsymbol v}} 

\newcommand{\calC}{\mathcal{C}} 
\newcommand{\calK}{\mathcal{K}} 

\newcommand{\calX}{\mathcal{X}} 
\newcommand{\calW}{\mathcal{W}} 

\newcommand{\calO}{\mathcal{O}} 
\newcommand{\calS}{\mathcal{S}} 

\DeclareMathOperator*{\argmin}{arg\,min}
\newcommand{\od}{\operatorname{d}}

\newtheorem{theorem}{Theorem}
\newtheorem{lemma}{Lemma}

\theoremstyle{definition}
\theoremstyle{definition}
\newtheorem{problem}{Problem}

\theoremstyle{definition}

\theoremstyle{definition}

\theoremstyle{definition}
\newtheorem{assumption}{Assumption}

\usepackage{fontawesome5}

%% file: content.tex
\providecommand{\calU}{\mathcal{U}}
\providecommand{\calR}{\mathcal{R}}
\providecommand{\calH}{\mathcal{H}}
\providecommand{\calI}{\mathcal{I}}
\providecommand{\calQ}{\mathcal{Q}}
\providecommand{\va}{{\boldsymbol a}}
\providecommand{\vd}{{\boldsymbol d}}
\providecommand{\veta}{{\boldsymbol \eta}}
\providecommand{\vxi}{{\boldsymbol \xi}}

\hypersetup{
    pdfborder={0 0 0}
}

\begin{abstract}
Robots navigating under occlusion may enter states from which no admissible input can avoid a dynamic obstacle once it becomes visible. We present OcclusionCBF, a safety filter that extends backup control barrier functions to reachable-occupancy predictions for potentially hidden dynamic obstacles in occluded regions. The method certifies a prescribed backup rollout against collision-inflated occupancy and a verified terminal set, yielding affine constraints for minimally invasive quadratic-program filtering. We establish recursive feasibility of the resulting safety filter, and collision avoidance for every hidden-obstacle motion covered by the occupancy prediction. Randomized benchmarks, MetaUrban simulations, and hardware experiments demonstrate improved task success over reactive and occlusion-aware predictive baselines with millisecond-scale computation. \href{https://www.taekyung.me/occlusion-cbf}{\textcolor{red}{[Project Page]}}\footnote{Project page: \href{https://www.taekyung.me/occlusion-cbf}{https://www.taekyung.me/occlusion-cbf}} \href{https://github.com/tkkim-robot/occlusion-cbf}{\textcolor{red}{[Code]}} \href{https://youtu.be/csfJzw1GwJo}{\textcolor{red}{[Video]}} \href{https://occlusion-cbf.taekyung.me/}{\textcolor{red}{[Web Demo]}}
\end{abstract}



\section{Introduction}
\label{sec:introduction}

Autonomous robots frequently operate with incomplete information about their surroundings.
Buildings, parked vehicles, shelving, and other objects create blind
regions in which pedestrians or vehicles may move without being observed.
A controller that avoids only currently detected obstacles may therefore
steer the robot into a state from which no admissible input can prevent a
collision once a hidden obstacle becomes visible. As illustrated in
\autoref{fig:ocbf_intro}, this failure is not merely a consequence of sensing latency. It arises from the interaction among uncertain hidden-obstacle occupancy, robot dynamics, and finite control authority.

A safety-critical controller must therefore reason about where a hidden
obstacle may be and how it may move before detection. Existing approaches
propagate hidden-obstacle positions under bounded motion and incorporate the
resulting reachable occupancy into safety verification or receding-horizon
planning~\cite{sanchez_foresee_2022,firoozi_oampc_2025,zheng_occlusionaware_2026}.
Although effective, embedding this prediction within the planner couples
safety to a particular planning architecture and can require solving a
trajectory optimization at every control update. Our objective is instead
to construct a modular safety filter that can be interposed between an
arbitrary nominal controller and the system dynamics while explicitly accounting for
the robot dynamics and input constraints.

\begin{figure}[t]
    \centering
    \includegraphics[width=0.99\linewidth]{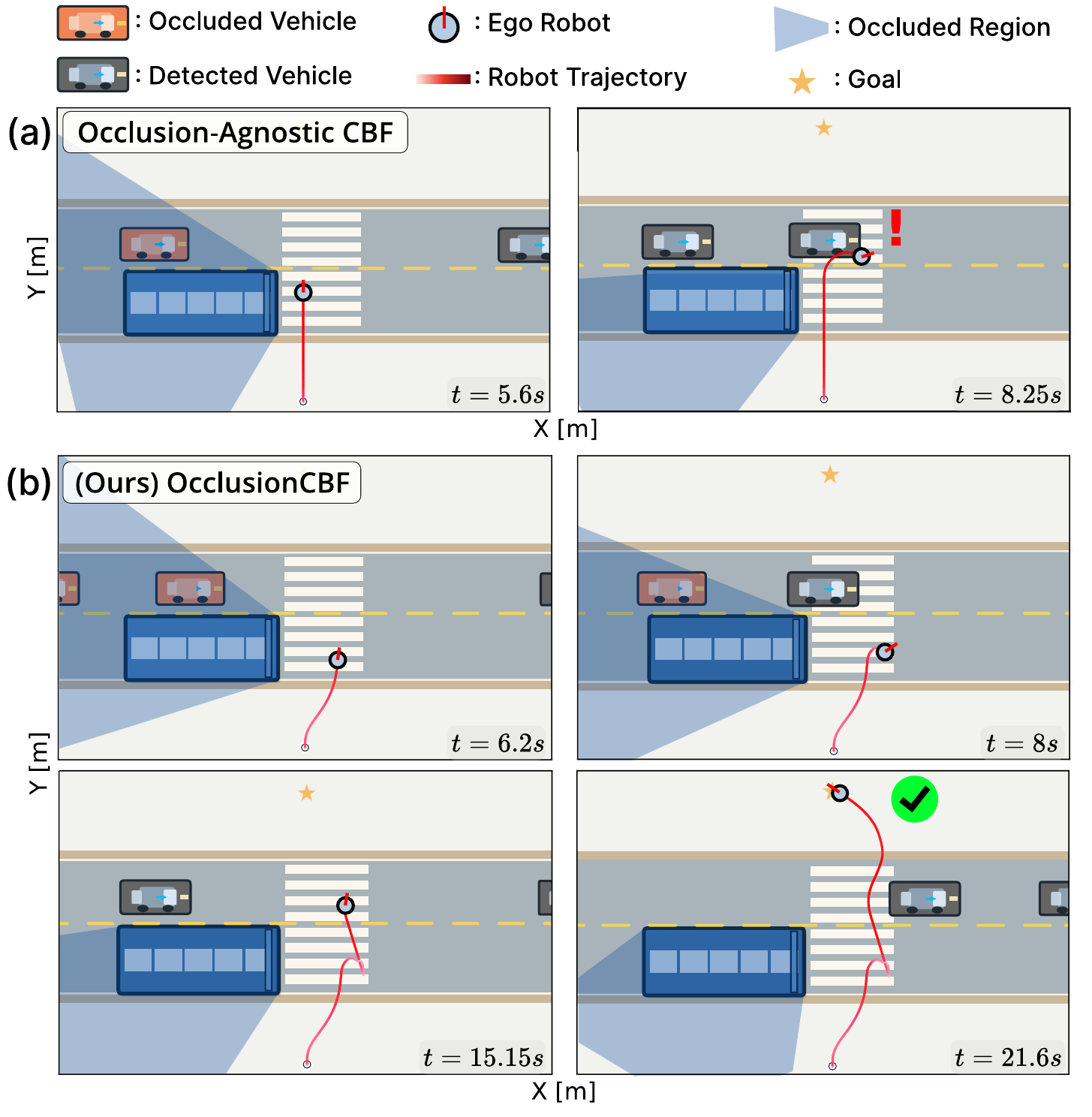}
    \caption{Illustrative comparison of safety filtering with and without
    occlusion awareness for a double-integrator robot at a blind crossing.
    (a) A conventional CBF-QP considers only detected vehicles and follows
    the nominal command while the approaching vehicle remains occluded,
    reaching a state with insufficient control authority to avoid collision
    after detection.
    (b) OcclusionCBF propagates the admissible motion of potentially hidden
    vehicles, certifies a backup rollout against the resulting reachable occupancy, and minimally modifies the nominal input to slow down and
    increase clearance before entering the crossing.}
    \label{fig:ocbf_intro}
\vspace{-5pt}
\end{figure}

We address this problem with \textbf{\emph{OcclusionCBF}}, an extension of
the Backup Control Barrier Function~(CBF) formulation \cite{chen_backup_2021}
to time-varying reachable-occupancy predictions for potentially hidden
dynamic obstacles. At each prediction update, possible obstacle locations are
propagated over the backup horizon and inflated for collision avoidance. A
prescribed backup policy is rolled out against every reachable-occupancy set
and required to terminate in a verified time-varying terminal set. By
differentiating the occupancy and terminal margins through the backup flow,
including their dependence on current time and look-ahead time, we obtain
constraints affine in the current control input. The resulting Quadratic
Program~(QP) safety filter minimally modifies the nominal input while
retaining the recursive-feasibility mechanism of Backup CBFs. Under the
stated occupancy-prediction, terminal-set, and update conditions, the backup
input certifies pointwise QP feasibility, while the resulting safety filter
preserves the recoverable set and avoids every hidden-obstacle motion covered
by the occupancy prediction.

\subsection{Related Work}

Visibility-aware planning constrains motion to observed free space or
actively seeks trajectories that reveal relevant regions before
traversal~\cite{goretkin_look_2020,kim_visibilityaware_2025}. These methods
address unknown geometry and limited sensing, but generally do not model the
worst-case future motion of an unseen dynamic obstacle. Set-based verification
identifies critical field-of-view boundaries, represents possible hidden
obstacles by interval-valued states, propagates their occupancy, and verifies
whether a given trajectory retains a collision-free fail-safe
maneuver~\cite{orzechowski_tackling_2018}. Sequential reachability models
hidden obstacles as point-mass state sets and refines them using reachability
and previous observations~\cite{sanchez_foresee_2022}. Game-theoretic
active perception instead formulates planning as a hybrid pursuit-evasion
game that accounts for future detection and feedback avoidance after
detection~\cite{zhang_safe_2021}.

Occlusion-aware predictive planners embed hidden-agent reachability and
contingency reasoning within trajectory optimization. OA-MPC combines
forward reachable sets with a terminal stopping condition to establish
recursive feasibility~\cite{firoozi_oampc_2025}. Control-Tree optimization
and OACP represent contingencies through branched controls over discrete
hypotheses and jointly optimized exploration and fallback trajectories,
respectively~\cite{phiquepal_controltree_2021,zheng_occlusionaware_2026}.
APRO represents occlusion reachability using unions of AH-polytopes and
reduces exact trajectory-safety checks to linear
programs~\cite{chung_exact_2026}. These methods perform occlusion-aware
safety reasoning at the trajectory level. OcclusionCBF instead
converts a compatible conservative reachable-occupancy prediction into an
input-affine safety filter for an arbitrary nominal controller.

CBFs enforce forward invariance through an affine constraint on the
immediate input~\cite{ames_control_2019}. High-order CBFs address
constraints with high relative degree~\cite{xiao_control_2019}, but do
not by themselves ensure feasibility under input constraints. Backup CBFs use
the recoverable set induced by a backup policy as an implicit safe
set~\cite{chen_backup_2021,kim_backupbased_2026}, while robust Backup CBFs
address uncertainty in the ego-system
dynamics~\cite{vanwijk_disturbancerobust_2024}. OcclusionCBF instead propagates uncertainty in exogenous hidden-obstacle
occupancy and evaluates the backup trajectory against the resulting
occupancy predictions.

\subsection{Contributions}
Our contributions are summarized as follows.
\begin{itemize}[leftmargin=*,itemsep=1pt,topsep=2pt]
    \item We extend Backup CBFs to conservative reachable-occupancy
    predictions, yielding a planner-agnostic filter that certifies a backup
    maneuver for potentially hidden dynamic obstacles without embedding
    safety in a specific trajectory optimizer.

    \item We derive input-affine rollout and terminal constraints for the
    moving prediction horizon and prove that the backup policy certifies recursive feasibility of the QP and collision avoidance for every
    obstacle motion covered by the occupancy prediction.

    \item We give a smooth polyhedral realization for multiple occluded
    regions. Across five baselines, randomized tests attain the highest
    hard-constrained success rate at every tested density with
    millisecond-scale computation; MetaUrban and hardware results show
    intervention before detection and safe goal completion.
\end{itemize}

\section{PRELIMINARIES}
\label{sec:preliminaries}

\subsection{Robot Dynamics and Time-Varying Safety}

Consider a robot with control-affine dynamics
\begin{equation}
    \dot{\vx}=f(\vx)+g(\vx)\vu,
    \qquad
    \vx\in\calX,\quad \vu\in\calU,
    \label{eq:system}
\end{equation}
where $\calX\subseteq\mathbb{R}^{n}$ and
$\calU\subset\mathbb{R}^{m}$ are compact and convex, and $f$ and $g$ are
continuously differentiable and locally Lipschitz. Let
$\calW\subseteq\mathbb{R}^{d}$ be the workspace. The continuously
differentiable function $P:\calX\rightarrow\calW$ extracts the workspace position
$\vp=P(\vx)$ of a fixed robot reference point from the robot state. The robot is
commanded toward its objective by a possibly time-varying nominal
controller
$\vu_{\textup{nom}}:\calX\times\mathbb{R}_{\geq0}\rightarrow\calU$. A safety filter modifies $\vu_{\textup{nom}}$ to satisfy safety constraints;
it is minimally invasive when it selects, among the feasible inputs, one
closest to $\vu_{\textup{nom}}$ in a prescribed norm.

We use time-varying barrier functions to enforce safety constraints induced by moving occupancy sets. Let
$h:\calX\times\mathbb{R}_{\geq 0}\rightarrow\mathbb{R}$ be continuously
differentiable and define
$\calC(t)=\{\vx\in\calX\mid h(\vx,t)\geq 0\}$. Suppose there exists an
extended class-$\calK$ function $\alpha:\mathbb{R}\rightarrow\mathbb{R}$ for which, for all
$(\vx,t)\in\calX\times\mathbb{R}_{\geq0}$,
\begin{equation}
\begin{aligned}
&\sup_{\vu\in\calU}\Biggl[
    \frac{\partial h}{\partial t}(\vx,t)
    +\nabla_{\vx}h(\vx,t)^{\top}
    \bigl(f(\vx)+g(\vx)\vu\bigr)
\Biggr] \\
&\qquad \geq -\alpha\bigl(h(\vx,t)\bigr).
\end{aligned}
\label{eq:tvcbf_condition}
\end{equation}
A minimally invasive safety filter can be implemented as the CBF-QP
\[
\begin{aligned}
& \vu^{\star}(\vx,t) =
\argmin_{\vu\in\calU}
\left\|\vu-\vu_{\textup{nom}}(\vx,t)\right\|_{2}^{2}
\\
& \textup{s.t.} \, 
\frac{\partial h}{\partial t}(\vx,t)
+\nabla_{\vx}h(\vx,t)^{\top}
\bigl(f(\vx)+g(\vx)\vu\bigr)
\geq-\alpha\bigl(h(\vx,t)\bigr).
\end{aligned}
\]
Any locally Lipschitz controller satisfying this CBF constraint renders $\calC(t)$ forward invariant, provided
$\vx(t_{0})\in\calC(t_{0})$~\cite{garg_advances_2024}.
Importantly, $h(\vx,t)\geq0$ alone does not guarantee feasibility
of this constraint under input limits~\cite{kim_how_2025}.

\subsection{Backup Control Barrier Functions}
\label{sec:backup_cbf}

For the standard Backup CBF
construction~\cite{chen_backup_2021,kim_backupbased_2026},
consider the time-invariant special case $\calC(t)\equiv\calC$.
Let $\calS_{0}\subseteq\calC$ be a terminal set forward invariant
under a continuously differentiable state-feedback law
$\pi_{\textup{b}}:\calX\rightarrow\calU$, and let $T>0$ be a
finite backup horizon. The states whose $\pi_{\textup{b}}$-rollouts
remain in $\calC$ over $[0,T]$ and reach $\calS_{0}$ at $T$ form
the recoverable set induced by $\pi_{\textup{b}}$.
Such a policy is referred to as a \emph{backup policy}, and its choice
affects the size of the induced recoverable set.
Practical designs include saturated lane-keeping
feedback~\cite{chen_backup_2021} and PD feedback toward a designated recovery
region~\cite{kim_backupbased_2026}.

The closed-loop backup dynamics are
$f_{\textup{b}}(\vx)=f(\vx)+g(\vx)\pi_{\textup{b}}(\vx)$.
Their flow map is denoted by $\varphi_{\textup{b}}(\vx,s)$, where $s\geq0$
is the backup-rollout time and
\begin{equation}
\frac{\od}{\od s}\varphi_{\textup{b}}(\vx,s)
=f_{\textup{b}}\big(\varphi_{\textup{b}}(\vx,s)\big),
\qquad
\varphi_{\textup{b}}(\vx,0)=\vx.
\label{eq:backup_flow}
\end{equation}
The sensitivity Jacobian
$\Phi_{\textup{b}}(\vx,s):=
\partial\varphi_{\textup{b}}(\vx,s)/\partial\vx$
satisfies
\begin{equation}
\frac{\od}{\od s}\Phi_{\textup{b}}(\vx,s)
=
\frac{\partial f_{\textup{b}}}{\partial\vx}
\big(\varphi_{\textup{b}}(\vx,s)\big)
\Phi_{\textup{b}}(\vx,s),
\quad
\Phi_{\textup{b}}(\vx,0)=I.
\label{eq:backup_sensitivity_dynamics}
\end{equation}
For a fixed choice of $\pi_{\textup{b}}$, $f_{\textup{b}}$ is autonomous.
Hence,
$\varphi_{\textup{b}}\big(\varphi_{\textup{b}}(\vx,r),s\big)
=\varphi_{\textup{b}}(\vx,r+s)$ for $r,s\geq0$. Differentiating this flow
property with respect to $r$ at $r=0$ yields
\begin{equation}
\Phi_{\textup{b}}(\vx,s)f_{\textup{b}}(\vx)
=
f_{\textup{b}}\big(\varphi_{\textup{b}}(\vx,s)\big).
\label{eq:flow_identity}
\end{equation}

\section{PROBLEM FORMULATION}
\label{sec:problem}

Consider a robot with dynamics in \eqref{eq:system}, commanded by a nominal controller $\vu_{\textup{nom}}$ and equipped with a sensor with limited and potentially occluded field of view. We seek to modify the nominal controller to ensure collision avoidance against dynamic obstacles that may
remain unobserved until they enter the sensor's visible region. Assume every
dynamic obstacle has a known speed bound $\bar v_{\textup{o}}>0$, so its
displacement over any interval of length $s\in[0,T]$ is contained in
$\mathcal{B}(\bar v_{\textup{o}}s)$, where
$\mathcal{B}(r)\subseteq\mathbb{R}^{d}$ is the closed Euclidean ball of
radius $r$.

At time $t$, let
$\calO(P(\vx(t)),t)\subseteq\calW$ denote the set of workspace positions
at which an undetected obstacle may be located, consistently with the current
sensing geometry and observations. An occupancy predictor represents this
set as a finite union of occluded regions:
\begin{equation}
\calO\big(P(\vx(t)),t\big)
=
\bigcup_{j=1}^{M(t)} 
\calO^{(j)}\big(P(\vx(t)),t\big),
\label{eq:occluded_region_union}
\end{equation}
where $M(t)$ is the number of occluded regions returned by the
predictor. Hereafter, we write $M$ when its dependence on $t$ is clear. For each
$j=1,\ldots,M$,
the predictor returns a collision-inflated reachable-occupancy set
$\widehat{\calH}^{(j)}(t,s)$ satisfying
\begin{equation}
\calO^{(j)}\big(P(\vx(t)),t\big)
\oplus
\mathcal{B}\!\left(
\bar v_{\textup{o}}s+r_{\textup{col}}
\right)
\subseteq
\widehat{\calH}^{(j)}(t,s).
\label{eq:inflated_occupancy}
\end{equation}
Here, $\oplus$ denotes the Minkowski sum, and $r_{\textup{col}}\geq0$ is the collision-inflation radius
accounting for the robot and obstacle footprints, localization and
perception uncertainty, and any prescribed clearance.

Let $\{t_k\}_{k\geq0}$ be the times at which new observations are
incorporated and the occupancy prediction is recomputed. On each interval
$[t_k,t_{k+1})$, $M$ and the set indexing are fixed.

\begin{assumption}[Occupancy coverage and temporal consistency]
\label{ass:prediction}
For every $t\geq0$, $s\in[0,T]$, and $j=1,\ldots,M$,
\eqref{eq:inflated_occupancy} holds. Moreover, on each
interval $[t_k,t_{k+1})$, for every $\delta\geq0$ such that
$t,t+\delta\in[t_k,t_{k+1})$ and $s,s+\delta\in[0,T]$,
\begin{equation}
\widehat{\calH}^{(j)}(t+\delta,s)
\subseteq
\widehat{\calH}^{(j)}(t,s+\delta).
\label{eq:set_consistency}
\end{equation}
\end{assumption}

Condition~\eqref{eq:set_consistency} compares predictions for the same
absolute future time because $(t+\delta)+s=t+(s+\delta)$, i.e., it prevents a
previously certified occupancy set from enlarging as the horizon shifts. Any discontinuous recomputation at $t_k$ requires post-update recertification of the current state.

\begin{problem}
Given the robot dynamics in \eqref{eq:system}, the nominal controller
$\vu_{\textup{nom}}$, and an occupancy predictor satisfying Assumption~\ref{ass:prediction}, construct a minimally invasive safety
filter that remains feasible at all subsequent times on its certified
recoverable set under $\vu\in\calU$ and guarantees collision avoidance for
every hidden-obstacle motion covered by the occupancy prediction.
\label{problem:occlusion-cbf}
\end{problem}

\begin{figure}[t]
    \centering
    \includegraphics[width=0.99\linewidth]{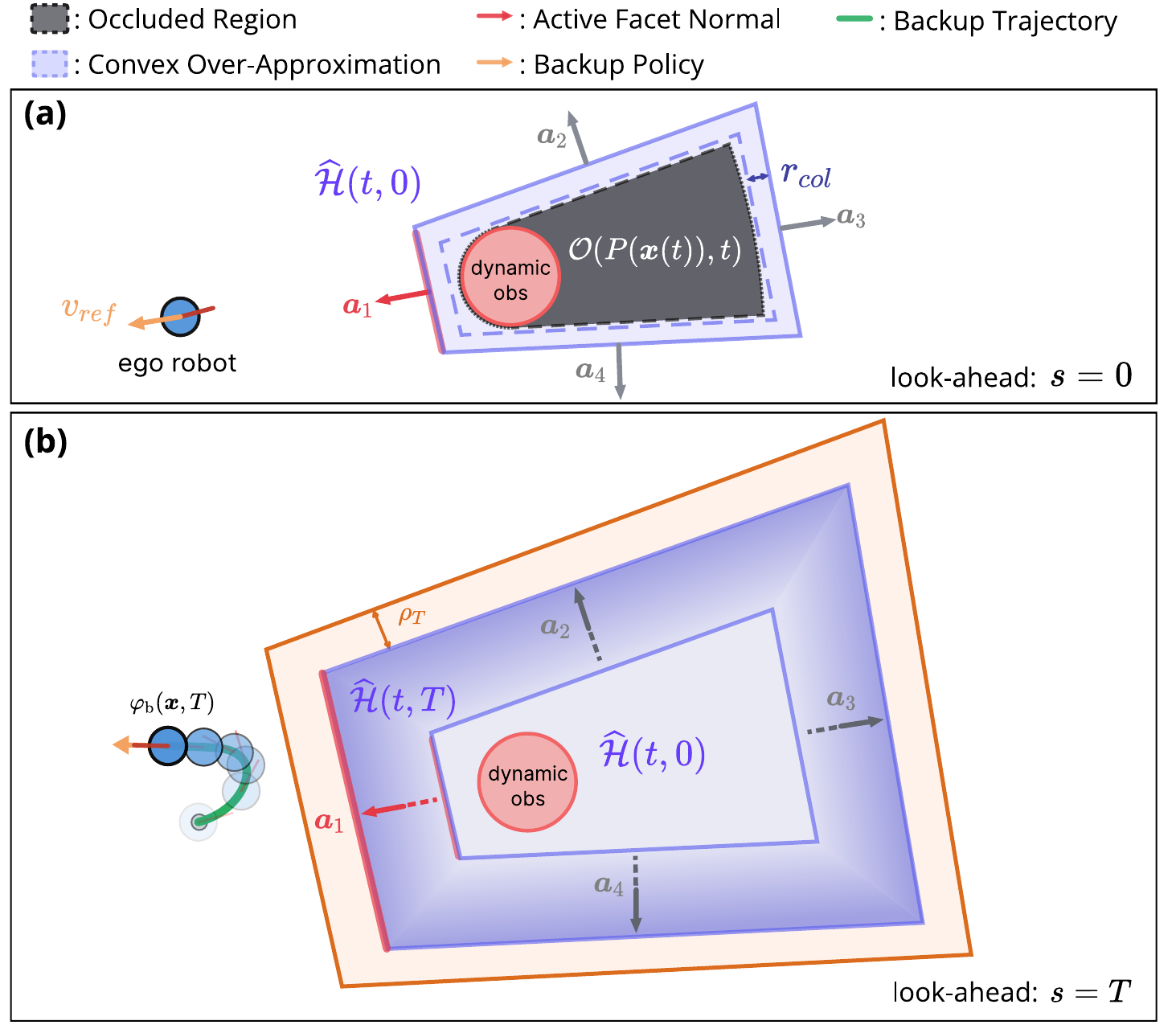}
    \caption{Geometric interpretation for one occluded region.
    (a) The occluded region is enclosed by
    $\widehat{\calO}^{(j)}(t)$ and inflated by $r_{\textup{col}}$ to
    $\widehat{\calH}^{(j)}(t,0)$; $h_j^C(\vy,t,0)$ is constructed from its
    facet margins. (b) The occupancy expands with the look-ahead time $s$ while the
    backup rollout remains outside it and reaches $\calS_0(t+T)$ with terminal
    clearance $\rho_T$.}
    \label{fig:ocbf_geometry}
\vspace{-5pt}
\end{figure}

\section{OCCLUSION CBF}
\label{sec:method}

Physical separation from the predicted occupancy at the current time does not
by itself ensure recoverability, which depends jointly on the robot dynamics,
input limits, and the future evolution of the hidden-obstacle occupancy. We
therefore develop \textbf{\emph{OcclusionCBF}} to certify a finite-horizon backup maneuver
that remains outside the time-varying reachable occupancy and terminates in a
verified terminal set. Preserving this certificate ensures recoverability and
collision avoidance for all hidden-obstacle motions covered by the prediction.

\subsection{Polyhedral Reachable-Occupancy Margin}
\label{sec:polytope}

Let $\vy\in\calX$ denote a dummy state variable along the backup rollout, so that $\vy=\varphi_{\textup{b}}(\vx,s)$.
For each reachable-occupancy set, let
$h_{j}^{C}:\calX\times\mathbb{R}_{\geq0}\times[0,T]\rightarrow\mathbb{R}$
be a differentiable conservative separation margin satisfying
\begin{equation}
h_{j}^{C}(\vy,t,s)\geq0
\quad\Longrightarrow\quad
P(\vy)\notin
\operatorname{int}\!\left(
\widehat{\calH}^{(j)}(t,s)
\right).
\label{eq:margin_certificate}
\end{equation}
To ensure temporal consistency of the occupancy margin~$h_{j}^{C}$ under the
advancing prediction horizon (as $t$ increases and the corresponding
look-ahead time $s$ decreases; cf. Assumption~\ref{ass:prediction}), we require, on each interval
$[t_k,t_{k+1})$,
\begin{equation}
\frac{\partial h_{j}^{C}}{\partial t}(\vy,t,s)
-
\frac{\partial h_{j}^{C}}{\partial s}(\vy,t,s)
\geq0.
\label{eq:differential_consistency}
\end{equation}
The occluded regions $\calO^{(j)}$ are supplied by a sensing or
geometric-processing module and need not be polyhedral, as
illustrated in \autoref{fig:ocbf_geometry}.
We construct conservative polyhedral outer approximations and
derive $h_j^C$ from their facet margins.
For polygonal regions, such an approximation can be obtained
by applying a standard convex-hull algorithm to their vertices.

At an update time $t_k$, enclose the $j$-th occluded region by the convex
polytope
\begin{equation}
\begin{aligned}
& \calO^{(j)}\big(P(\vx(t_k)),t_k\big)
\subseteq\widehat{\calO}^{(j)}(t_k) \\
= & \Bigl\{\vq\in\mathbb{R}^{d}\,\Bigm| \va_{j\ell}(t_k)^{\top}\vq
\leq b_{j\ell}(t_k), \ell=1,\ldots,K_j\Bigr\},
\end{aligned}
\label{eq:predictor_polytope}
\end{equation}
where $\|\va_{j\ell}(t_k)\|_{2}=1$. For
$t\in[t_k,t_{k+1})$, retain $K_j$ and set
$\va_{j\ell}(t)=\va_{j\ell}(t_k)$ and
$b_{j\ell}(t)=b_{j\ell}(t_k)+\bar v_{\textup{o}}(t-t_k)$.
The same half-space form then defines $\widehat{\calO}^{(j)}(t)$ as the
$t_k$ occupancy estimate propagated by the elapsed-time motion bound. Adding the look-ahead expansion
and collision radius gives
\begin{equation}
\begin{split}
\raisetag{3.0ex}
\widehat{\calH}^{(j)}(t,s)
=
\Big\{
\vq\in\mathbb{R}^{d}\ \Big|\
&\va_{j\ell}(t)^{\top}\vq
\leq b_{j\ell}(t)+r_{\textup{col}}+\bar v_{\textup{o}}s,\\
&\ell=1,\ldots,K_j
\Big\}.
\end{split}
\label{eq:expanded_polytope}
\end{equation}
The elapsed-time expansion in $b_{j\ell}(t)$ covers motion since $t_k$,
while $\bar v_{\textup{o}}s+r_{\textup{col}}$ covers future motion and
collision clearance; hence \eqref{eq:expanded_polytope} satisfies
\eqref{eq:inflated_occupancy}. In general, it is a polyhedral outer
approximation of the corresponding Minkowski sum. Because
$b_{j\ell}(t+\delta)=b_{j\ell}(t)+\bar v_{\textup{o}}\delta$ between
updates,
$\widehat{\calH}^{(j)}(t+\delta,s)
=\widehat{\calH}^{(j)}(t,s+\delta)$, so
\eqref{eq:set_consistency} holds with equality.

For a robot state $\vy$, define the facet margins
\begin{equation}
\psi_{j\ell}(\vy,t,s)
=
\va_{j\ell}(t)^{\top}P(\vy)
-b_{j\ell}(t)
-r_{\textup{col}}
-\bar v_{\textup{o}}s.
\label{eq:face_margin}
\end{equation}
The robot position lies outside the interior of
$\widehat{\calH}^{(j)}(t,s)$ if and only if
\[
\max_{\ell=1,\ldots,K_j}
\psi_{j\ell}(\vy,t,s)
\geq0.
\]
To obtain a differentiable sufficient condition, define the shifted
log-sum-exp margin~\cite{molnar_composing_2023}
\begin{equation}
h_{j}^{C}(\vy,t,s)
=
\frac{1}{\kappa}
\log\left(
\sum_{\ell=1}^{K_{j}}
\exp\big(
\kappa\psi_{j\ell}(\vy,t,s)
\big)
\right)
-\frac{\log K_{j}}{\kappa},
\label{eq:smooth_margin}
\end{equation}
where $\kappa>0$. This function satisfies
\[
\max_{\ell}\psi_{j\ell}
-\frac{\log K_j}{\kappa}
\leq
h_{j}^{C}
\leq
\max_{\ell}\psi_{j\ell}.
\]
The bound quantifies the smoothing conservatism as
$0\leq\max_{\ell}\psi_{j\ell}-h_j^C\leq\log K_j/\kappa$; hence,
$h_j^C$ satisfies \eqref{eq:margin_certificate}. Separation is imposed
for every $j=1,\ldots,M$.

For notational compactness, define the softmax weights
\begin{equation}
\lambda_{j\ell}(\vy,t,s)
=
\frac{
\exp\big(
\kappa\psi_{j\ell}(\vy,t,s)
\big)
}{
\sum_{r=1}^{K_j}
\exp\big(
\kappa\psi_{jr}(\vy,t,s)
\big)
},
\qquad
\sum_{\ell=1}^{K_j}\lambda_{j\ell}=1. \nonumber
\label{eq:softmax_weight}
\end{equation}
Between updates, the derivatives used below are
\begin{subequations}
\label{eq:smooth_derivatives}
\begin{align}
\nabla_{\vy}h_{j}^{C}
&=
\left(
\frac{\partial P(\vy)}{\partial\vy}
\right)^{\top}
\sum_{\ell=1}^{K_j}
\lambda_{j\ell}\va_{j\ell},
\\
\frac{\partial h_{j}^{C}}{\partial s}
&=
\frac{\partial h_{j}^{C}}{\partial t}
=
-\bar v_{\textup{o}}.
\label{eq:facet_offset_evolution}
\end{align}
\end{subequations}
The last equality follows from
$\dot{\va}_{j\ell}=0$ and
$\dot b_{j\ell}=\bar v_{\textup{o}}$. Consequently,
\eqref{eq:differential_consistency} also holds with equality. Changes in facet number or geometry at $t_k$ are handled by
recertifying the current state after the update.

\subsection{Backup Rollout and Terminal Set}
\label{sec:backup_terminal}

Having constructed differentiable occupancy margins, we next adapt the
Backup CBF formulation in \autoref{sec:backup_cbf} to the
moving-horizon setting of Problem~\ref{problem:occlusion-cbf}. The backup rollout must remain outside every time-varying reachable-occupancy set and terminate in a time-varying set that is forward invariant under the backup policy.

On each update interval $[t_k,t_{k+1})$, the backup policy
$\pi_{\textup{b}}$ is held fixed, with corresponding closed-loop flow
$\varphi_{\textup{b}}$, although it may change at the next update. We
suppress this interval dependence in the notation and use the same policy
for all simultaneously enforced occupancy sets.
For each set index $j$, let
$h_{j}^{S}:\calX\times[T,\infty)\rightarrow\mathbb{R}$ be a
continuously differentiable terminal barrier satisfying, for
all $\vy\in\calX$, $q\geq T$ and $j=1,\ldots,M$,
\begin{equation}
h_{j}^{S}(\vy,q)\geq0 \quad 
\Longrightarrow \quad
h_{j}^{C}(\vy,q-T,T)\geq\rho_{T},
\label{eq:terminal_barrier}
\end{equation}
where $\rho_{T}\geq0$ is an additional terminal clearance
(see \autoref{sec:simulation} for a model-specific selection).
The corresponding time-varying terminal set is
$\calS_{0}(q)=
\{\vy\in\calX\mid h_j^S(\vy,q)\geq0,\ j=1,\ldots,M\}$.

\begin{assumption}[Verified terminal set]
\label{ass:terminal}
There exists an extended class-$\calK$ function $\alpha_{S}$ such that,
for every $q\geq T$, every $\vy\in\calS_{0}(q)$, and every
$j=1,\ldots,M$,
\begin{equation}
\frac{\partial h_{j}^{S}}{\partial q}(\vy,q)
+
\big(
\nabla_{\vy}h_{j}^{S}(\vy,q)
\big)^{\top}
f_{\textup{b}}(\vy)
\geq
-\alpha_{S}\big(
h_{j}^{S}(\vy,q)
\big).
\label{eq:terminal_invariance}
\end{equation}
\end{assumption}

Assumption~\ref{ass:terminal} renders
$\calS_{0}(q)$ forward invariant under the backup policy between
discontinuous prediction updates. To express the terminal condition in
terms of the current state, define the endpoint margins
\begin{equation}
\eta_{j}(\vx,t)
:=
h_{j}^{S}\big(
\varphi_{\textup{b}}(\vx,T),t+T
\big),
\qquad
j=1,\ldots,M.
\label{eq:eta_definition}
\end{equation}
Thus, $\varphi_{\textup{b}}(\vx,T)\in\calS_0(t+T)$ if and only if
$\eta_j(\vx,t)\geq0$ for all $j=1,\ldots,M$.
Moreover, by \eqref{eq:terminal_barrier},
\begin{equation}
\eta_{j}(\vx,t)\geq0
\,\Longrightarrow\,
h_{j}^{C}\big(
\varphi_{\textup{b}}(\vx,T),t,T
\big)
\geq
\rho_{T}
\geq0.
\label{eq:terminal_compatibility}
\end{equation}
Thus, the terminal condition implies the corresponding occupancy
constraint at the end of the backup horizon and, when $\rho_T>0$,
retains additional clearance from the horizon occupancy. The occupancy and terminal constructions are illustrated in
\autoref{fig:ocbf_geometry}.

\subsection{Occlusion-Aware Recoverable Set and OCBF-QP}

We now define the time-varying recoverable set induced by the
backup policy and enforce its occupancy and terminal
conditions through a QP.

For $M\geq1$, define the occlusion-aware recoverable set
\begin{equation}
\begin{split}
\raisetag{3.0ex}
\calS(t)=\big\{\vx\in\calX\ \big|\ &
h_{j}^{C}\big(
\varphi_{\textup{b}}(\vx,s),t,s
\big)\geq0,
\quad
\forall s\in[0,T],\ \forall j,
\\
&\eta_{j}(\vx,t)\geq0,
\quad
\forall j
\big\}.
\end{split}
\label{eq:recoverable_set}
\end{equation}

We impose separate CBF inequalities on the occupancy and terminal
conditions defining $\calS(t)$.
For an absolute future time
$\tau\in[t,t+T]$, define
\begin{equation}
\begin{aligned}
s&=\tau-t,
\qquad
\vy=\varphi_{\textup{b}}(\vx,s),
\\
\xi_{j}^{C}(\vx,t;\tau)
&=
h_{j}^{C}(\vy,t,s).
\end{aligned}
\label{eq:xi_definition}
\end{equation}
Because $\od s/\od t=-1$ when $\tau$ is fixed,
\begin{equation}
\begin{split}
\dot{\xi}_{j}^{C}
=&\
\big(
\nabla_{\vy}h_{j}^{C}
\big)^{\top}
\left[
\Phi_{\textup{b}}(\vx,s)
\big(
f(\vx)+g(\vx)\vu
\big)
-f_{\textup{b}}(\vy)
\right]
\\
&+
\frac{\partial h_{j}^{C}}{\partial t}
-
\frac{\partial h_{j}^{C}}{\partial s}.
\end{split}
\label{eq:xi_derivative}
\end{equation}
For the polyhedral occupancy model, \eqref{eq:smooth_derivatives} gives
$\partial h_j^C/\partial t=\partial h_j^C/\partial s$, so the final two
terms in \eqref{eq:xi_derivative} cancel. 

The terminal margins defined in
\eqref{eq:eta_definition} satisfy
\begin{equation}
\dot{\eta}_{j}
=
\big(
\nabla_{\vy}h_{j}^{S}
\big)^{\top}
\Phi_{\textup{b}}(\vx,T)
\big(
f(\vx)+g(\vx)\vu
\big)
+
\frac{\partial h_{j}^{S}}{\partial q},
\label{eq:eta_derivative}
\end{equation}
where $h_j^S$ and its derivatives are evaluated at
$\big(\varphi_{\textup{b}}(\vx,T),t+T\big)$.

For notational compactness, define
\begin{equation}\label{eq:affine_terms}
\begin{split}
\raisetag{5.0ex}
A_{j}^{C}
&=
\big(
\nabla_{\vy}h_{j}^{C}
\big)^{\top}
\Phi_{\textup{b}}(\vx,s)g(\vx),
\\
c_{j}^{C}
&=
\big(
\nabla_{\vy}h_{j}^{C}
\big)^{\top}
\left[
\Phi_{\textup{b}}(\vx,s)f(\vx)
-f_{\textup{b}}(\vy)
\right]
+
\frac{\partial h_{j}^{C}}{\partial t}
-
\frac{\partial h_{j}^{C}}{\partial s}, 
\\
A_{j}^{S}
&=
\big(
\nabla_{\vy}h_{j}^{S}
\big)^{\top}
\Phi_{\textup{b}}(\vx,T)g(\vx), 
\\
c_{j}^{S}
&=
\big(
\nabla_{\vy}h_{j}^{S}
\big)^{\top}
\Phi_{\textup{b}}(\vx,T)f(\vx)
+
\frac{\partial h_{j}^{S}}{\partial q}. 
\end{split}
\end{equation}
In \eqref{eq:affine_terms}, the occupancy quantities are evaluated at
$(\vy,t,s)$ with $s=\tau-t$, while the terminal quantities are
evaluated at
$\big(\varphi_{\textup{b}}(\vx,T),t+T\big)$. Their dependence on
$(\vx,t;\tau)$ and $(\vx,t)$ is suppressed, respectively.

Let $\alpha_{C}$ be a locally Lipschitz extended class-$\calK$
function. The proposed Occlusion CBF QP~(OCBF-QP) is
\begin{equation}
\begin{split}
\raisetag{11.0ex}
\vu^{\star}(\vx,t)
=&\
\argmin_{\vu\in\calU}
\frac{1}{2}
\left\|
\vu-\vu_{\textup{nom}}(\vx,t)
\right\|_{W}^{2}
\\
\textup{s.t.}\quad
&
c_{j}^{C}+A_{j}^{C}\vu
\geq
-\alpha_{C}\big(\xi_{j}^{C}\big),
\quad
\forall j,\ 
\forall\tau\in[t,t+T],
\\
&
c_{j}^{S}+A_{j}^{S}\vu
\geq
-\alpha_{S}\big(\eta_{j}\big),
\quad
\forall j,
\end{split}
\label{eq:oabcbf_qp}
\end{equation}
where $W\succ0$. All constraints are affine in $\vu$.

\begin{figure*}[t]
    \centering
    \includegraphics[width=0.99\linewidth]{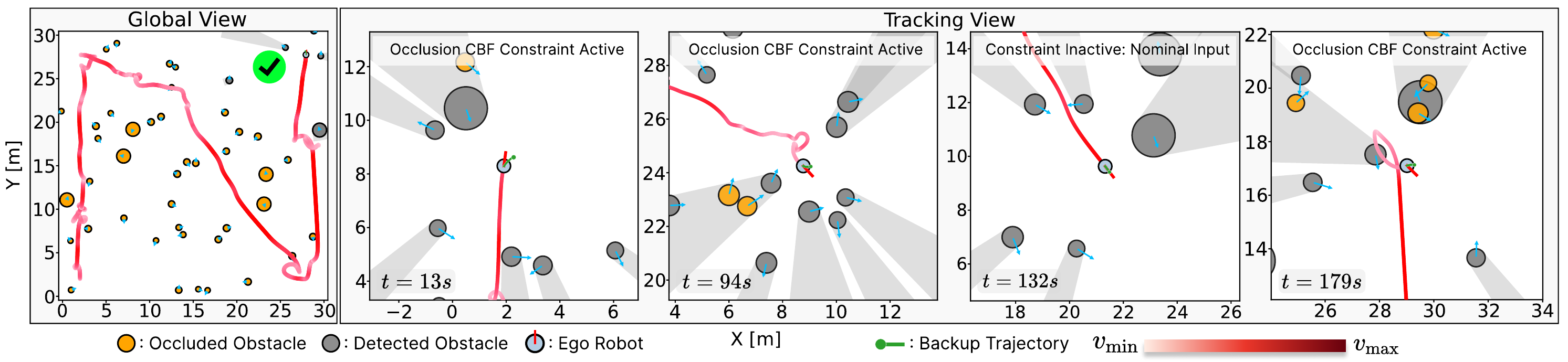}
    \caption{Representative successful OCBF-QP trial for the double
    integrator with $N_{\textup{obs}}=50$. The global view shows the
    completed waypoint route, while the local views show representative
    occlusion interactions and backup rollouts.}
    \label{fig:ocbf_snapshot}
    \vspace{-5pt}
\end{figure*}

\subsection{Recursive Feasibility and Safety Guarantees}

At an update time $t_k$, the occupancy prediction and backup-policy
parameters may change discontinuously, and hence so may the recoverable
set. We therefore require the continuous physical state to satisfy
\begin{equation}
\vx(t_k)\in\calS(t_k^{+}).
\label{eq:update_membership}
\end{equation}
This is a pointwise condition on the realized state and does not require
$\calS(t_k^{-})\subseteq\calS(t_k^{+})$.

The following lemma establishes pointwise feasibility of the OCBF-QP on
the recoverable set. The subsequent theorem shows that the QP constraints
preserve this set, yielding recursive feasibility and safety.

\begin{lemma}[Pointwise feasibility of the OCBF-QP]
\label{lem:qp_feasibility}
Suppose Assumptions~\ref{ass:prediction} and~\ref{ass:terminal} hold,
and the occupancy margins satisfy
\eqref{eq:differential_consistency}. Assume that $f$, $g$, $P$, $\pi_{\textup{b}}$, $h_j^C$, and $h_j^S$
are continuously differentiable between discontinuous updates, that
$\pi_{\textup{b}}(\vx)\in\calU$, and that the backup flow exists over
$[0,T]$. Then, for every time $t$ between updates and every
$\vx\in\calS(t)$, the backup input
$\vu=\pi_{\textup{b}}(\vx)$ satisfies all constraints of
\eqref{eq:oabcbf_qp}. Consequently, the OCBF-QP is feasible at every
state in $\calS(t)$.
\end{lemma}

\begin{proof}
Fix $\vx\in\calS(t)$ and choose
$\vu=\pi_{\textup{b}}(\vx)$. For every
$\tau\in[t,t+T]$, the flow identity~\eqref{eq:flow_identity} cancels the
state-dependent terms in~\eqref{eq:xi_derivative}, giving
\[
\left.
\dot{\xi}_j^C
\right|_{\vu=\pi_{\textup{b}}(\vx)}
=
\frac{\partial h_j^C}{\partial t}
-
\frac{\partial h_j^C}{\partial s}
\geq0
\]
by~\eqref{eq:differential_consistency}. Since
$\xi_j^C\geq0$ for $\vx\in\calS(t)$, this satisfies every occupancy
constraint in~\eqref{eq:oabcbf_qp}. Similarly, $\eta_j(\vx,t)\geq0$ for every $j$ implies
$\varphi_{\textup{b}}(\vx,T)\in\calS_0(t+T)$ by the definition of $\calS_0$ in Section \ref{sec:backup_terminal}. Applying
\eqref{eq:flow_identity} to~\eqref{eq:eta_derivative} and invoking
\eqref{eq:terminal_invariance} yields
\[
\left.
\dot{\eta}_j
\right|_{\vu=\pi_{\textup{b}}(\vx)}
\geq
-\alpha_S(\eta_j),
\qquad
j=1,\ldots,M.
\]
Thus, the backup input also satisfies every terminal constraint, proving
the claim of Lemma \ref{lem:qp_feasibility}.
\end{proof}

\begin{theorem}[Recursive feasibility of the OCBF-QP]
\label{thm:main}
Suppose the conditions of Lemma~\ref{lem:qp_feasibility} hold. 
If $\vx(t_0)\in\calS(t_0)$ and~\eqref{eq:update_membership} holds at every
discontinuous update, then every piecewise locally Lipschitz control
satisfying the OCBF-QP constraints and yielding a forward-complete
solution satisfies, for all $t\geq t_0$ and $j=1,\ldots,M(t)$,
\begin{equation}
\vx(t)\in\calS(t),
\qquad
h_j^C(\vx(t),t,0)\geq0.
\label{eq:main_safety_result}
\end{equation}
Consequently, the OCBF-QP is recursively feasible, and the robot avoids
all hidden-obstacle motions covered by the occupancy prediction.
\end{theorem}

\begin{proof}
Consider an update-free interval beginning at $\bar t$, where
$\bar t=t_0$ or $\bar t=t_k^{+}$. By the initial-condition assumption when
$\bar t=t_0$, and by \eqref{eq:update_membership} when
$\bar t=t_k^{+}$, we have $\vx(\bar t)\in\calS(\bar t)$. Hence, by the
definition of $\calS(\bar t)$,
\begin{equation*}
\xi_{j}^{C}(\vx(\bar t),\bar t;\tau)\geq0,
\quad
\forall \tau\in[\bar t,\bar t+T],
\end{equation*}
and $\eta_j(\vx(\bar t),\bar t)\geq0$ for every $j$.

For each fixed absolute future time $\tau$, while
$\tau\in[t,t+T]$, the OCBF-QP enforces
\[
\dot{\xi}_j^C
\geq
-\alpha_C(\xi_j^C),
\qquad
\dot{\eta}_j
\geq
-\alpha_S(\eta_j).
\]
The scalar comparison argument therefore preserves every occupancy and
terminal margin that is initialized nonnegative.

It remains to initialize the occupancy constraints that enter through
the advancing end of the horizon. A future time $\tau$ enters at
$t=\tau-T$, where $s=T$. Since $\eta_j(\vx(t),t)\geq0$,
\eqref{eq:terminal_compatibility} gives
\[
\begin{aligned}
\xi_j^C(\vx(t),t;t+T)
&=
h_j^C\big(
\varphi_{\textup{b}}(\vx(t),T),t,T
\big)
\\
&\geq
\rho_T
\geq0.
\end{aligned}
\]
Thus, every occupancy margin is nonnegative when it enters the horizon
and remains nonnegative while it is enforced, and every terminal margin
also remains nonnegative. Hence, all conditions defining $\calS(t)$ are
preserved between updates. Condition~\eqref{eq:update_membership}
reinitializes the same argument after each discontinuous update.

Therefore, $\vx(t)\in\calS(t)$ for all $t\geq t_0$. By
Lemma~\ref{lem:qp_feasibility}, the OCBF-QP remains feasible along the
resulting trajectory, establishing recursive feasibility. Finally,
setting $s=0$ in~\eqref{eq:recoverable_set} and using
$\varphi_{\textup{b}}(\vx,0)=\vx$ gives
$h_j^C(\vx(t),t,0)\geq0$ for every $j$. Together with
\eqref{eq:margin_certificate} and
Assumption~\ref{ass:prediction}, this ensures that $P(\vx(t))$ remains
outside every collision-inflated occupancy covering the admissible
hidden-obstacle motions.
\end{proof}

\subsection{Practical Implementation}

In implementation, the look-ahead interval is discretized as $0=s_0<\cdots<s_N=T$, and the trajectory constraints are enforced at
$\tau_i=t+s_i$, $i=0,\ldots,N$. The continuous-time guarantees above
apply to the ideal constraint family over all $s\in[0,T]$; the implemented
QP enforces only the sampled constraints and therefore does not exclude
inter-sample violations. Related treatments of semi-infinite safety
constraints and time-discretization error are provided
in~\cite{cohen_safety_2026, knoedler_safety_2025, kim_policy_2026}.

\section{RESULTS}
\label{sec:simulation}

We evaluate OCBF-QP in randomized planar benchmarks, MetaUrban simulation~\cite{wu_metaurban_2025},
and hardware experiments. Representative
scenarios can also be tested interactively through our web demo.\footnote{Web Demo: \href{https://occlusion-cbf.taekyung.me/}{https://occlusion-cbf.taekyung.me/}} 


Unless stated otherwise, all constraints in
\eqref{eq:oabcbf_qp} are hard. We also report a
relaxed-terminal ablation to assess the conservatism introduced by the
terminal condition. Because $h_j^S$ defines an auxiliary recoverability
barrier rather than a direct collision-separation constraint, this
variant introduces penalized slack only in the terminal inequalities
induced by $h_j^S$; all trajectory occupancy constraints remain hard.
The relaxed variant is empirical and is not covered by
Theorem~\ref{thm:main}.

\subsection{Randomized Benchmark Setup}
\label{sec:planar_benchmark}

\textbf{Robot and environment models: }
The double integrator has state
$\vx=[\vp^{\top},\vv^{\top}]^{\top}$ and acceleration input
$\vu$, with $\dot{\vp}=\vv$, $\dot{\vv}=\vu$,
$\|\vu\|_{\infty}\leq a_{\max}$, and
$\|\vv\|_{2}\leq v_{\max}$, where
$a_{\max}=1.0\,\textup{m/s}^{2}$ and
$v_{\max}=1.0\,\textup{m/s}$. The unicycle has state
$\vx=[p_x,p_y,\theta]^{\top}$ and input $\vu=[v,\omega]^{\top}$,
with $\dot p_x=v\cos\theta$, $\dot p_y=v\sin\theta$,
$\dot\theta=\omega$, $0\leq v\leq v_{\max}$, and
$|\omega|\leq\omega_{\max}$, where
$v_{\max}=1.0\,\textup{m/s}$ and
$\omega_{\max}=0.8\,\textup{rad/s}$. Thus, the first model captures finite
braking authority, whereas the second captures
nonholonomic turning limits. Both robots have radius $0.25\,\textup{m}$.

All trials use a $30\,\textup{m}\times30\,\textup{m}$ workspace and the
waypoint sequence $(2,2)\rightarrow(2,28)\rightarrow(28,2)\rightarrow(28,28)\,\textup{m}$.
The robot has a $10\,\textup{m}$ sensing range, a $360^{\circ}$ field
of view, and line-of-sight occlusion. Once an obstacle becomes visible,
its position and velocity are provided to every controller.

Each trial contains $N_{\textup{obs}}$ smaller moving obstacles with
radii sampled uniformly from $[0.3,0.4]\,\textup{m}$ and one larger
dynamic occluder with radius sampled from
$[0.8,1.0]\,\textup{m}$. Initially hidden obstacles are sampled behind
another obstacle at $t=0$. We test $N_{\textup{obs}}\in\{10,20,30\}$ for the unicycle, and $N_{\textup{obs}}\in\{10,30,50\}$ for the double integrator.

\textbf{Occupancy prediction and backup policy: }
At each prediction update, each occluder-induced unobserved region is
enclosed by a convex polygon
$\widehat{\calO}^{(j)}(t_k)$. The corresponding time-indexed occupancy
$\widehat{\calH}^{(j)}(t,s)$ is constructed using
\eqref{eq:expanded_polytope} with $\bar v_{\textup{o}}=1.0\,\textup{m/s}$.

For the backup policy, we select at most $M_{\textup{act}}$ occupancy sets
with the smallest current margins
$h_j^C(\vx(t_k),t_k,0)$. Their regularized directions away from the
associated occluder centers are combined using softmax weights that
prioritize smaller margins, producing a reference velocity with magnitude at most $v_{\max}$. If no set is active, the
reference is zero. The active index set and backup-policy parameters are
held fixed until the next prediction update, so the backup remains a
time-independent state-feedback law on each interval. The double-integrator
backup tracks this reference through a smoothly saturated acceleration
controller, whereas the unicycle maps it to an admissible desired heading and
speed. We use $T=0.25\,\textup{s}$ for the double integrator and
$T=2.0\,\textup{s}$ for the unicycle. 
For the benchmark implementation, we use the practical terminal-margin
construction
$h_j^S(\vy,q)=h_j^C(\vy,q-T,T)-\rho_T$.
For the double integrator,
$\rho_T=v_{\max}^{2}/(2a_{\max})=0.5\,\textup{m}$; for the unicycle,
$\rho_T=0\,\textup{m}$.

\textbf{Benchmark protocol and metrics: }
For each robot model and obstacle density, we generate $100$ randomized
scenarios and replay the same realizations across all methods. All methods
use the same robot, sensing, collision-checking, and simulation models,
with $\Delta t=0.05\,\textup{s}$ and a $500\,\textup{s}$ time limit. Algorithms are implemented in Python and JAX, and the
method-specific hyperparameters of all controllers are tuned using
Optuna. Complete implementation, tuning, and configuration details are provided in the public code repository.

Each trial is classified as success, collision, or controller
infeasibility. Success requires completing all waypoints without collision
or infeasibility; the latter is recorded when the controller returns no
admissible input. We report outcome percentages and mean computation time
per control update.

\subsection{Compared Methods}

We compare OCBF-QP with five methods under the common benchmark above.
\emph{CBF-QP}~\cite{ames_control_2019} is a reactive safety filter that
constrains only currently visible obstacles and does not propagate
hidden occupancy; for the double-integrator model, we use a high-order
CBF formulation. \emph{OA-MPC}~\cite{firoozi_oampc_2025} incorporates
hidden-obstacle reachable sets into receding-horizon optimization and
uses a terminal stopping condition. Its original formulation assumes
static occluding environment geometry, whereas moving obstacles in our
benchmark can themselves create and remove occlusions; this model
mismatch should therefore be considered when interpreting its results.

\emph{Control-Tree MPC}~\cite{phiquepal_controltree_2021} optimizes a
branched policy over discrete hidden-obstacle hypotheses generated by
the common occlusion module. \emph{Single-Hypothesis MPC} removes this branch
structure and instead optimizes one trajectory against an aggregated
risk region~\cite{phiquepal_controltree_2021}. \emph{OACP}~\cite{zheng_occlusionaware_2026} jointly
optimizes exploration and fallback trajectories with a shared initial
segment and reachable-occupancy constraints.

\subsection{Benchmark Results}

\textbf{Qualitative behavior: }
Figure~\ref{fig:ocbf_snapshot} shows that OCBF-QP completes the waypoint
sequence in a dense double-integrator scenario while intervening only
when the nominal command would compromise the certified backup
maneuver. When the nominal input already satisfies the Occlusion CBF
constraints, it is applied without modification.

\textbf{Safety and task completion: }
Figure~\ref{fig:performance_results} summarizes the paired randomized
trials. OCBF-QP achieves the highest success rate among the
hard-constrained variants for both robot models and every tested
density. At the highest densities, it succeeds in $86\%$ of unicycle
trials and $70\%$ of double-integrator trials, while the compared
methods increasingly terminate through collision or controller
infeasibility. The relaxed-terminal ablation improves the denser
double-integrator cases but does not inherit the guarantee of
Theorem~\ref{thm:main}.

\begin{figure}[t]
    \centering
    \includegraphics[width=0.99\linewidth]{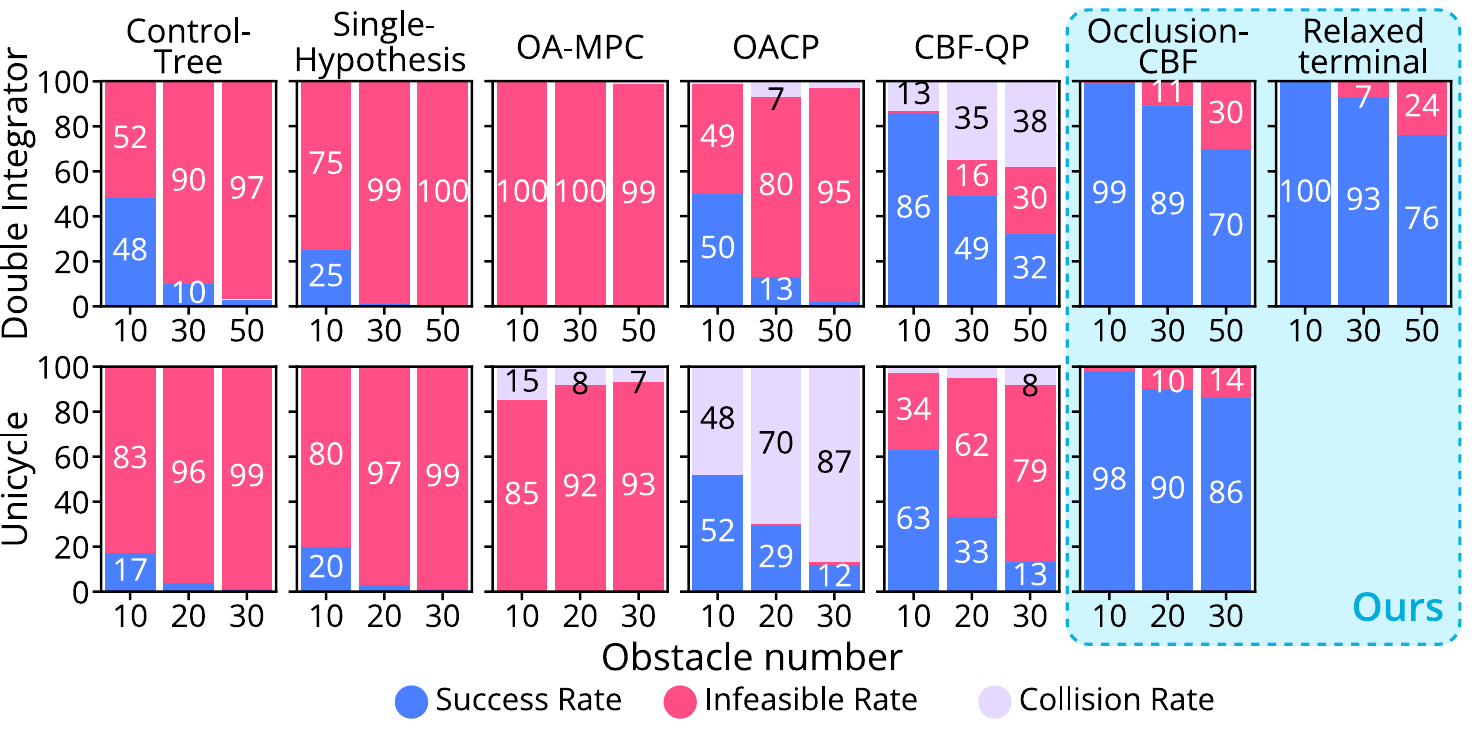}
    \caption{Benchmark results over $100$ paired trials for the double integrator
    (top) and unicycle (bottom) as $N_{\textup{obs}}$ increases. OCBF-QP attains the highest hard-constrained
    success rate at every tested density.}
    \label{fig:performance_results}
    \vspace{-5pt}
\end{figure}

\begin{figure*}[t]
    \centering    \includegraphics[width=0.99\linewidth]{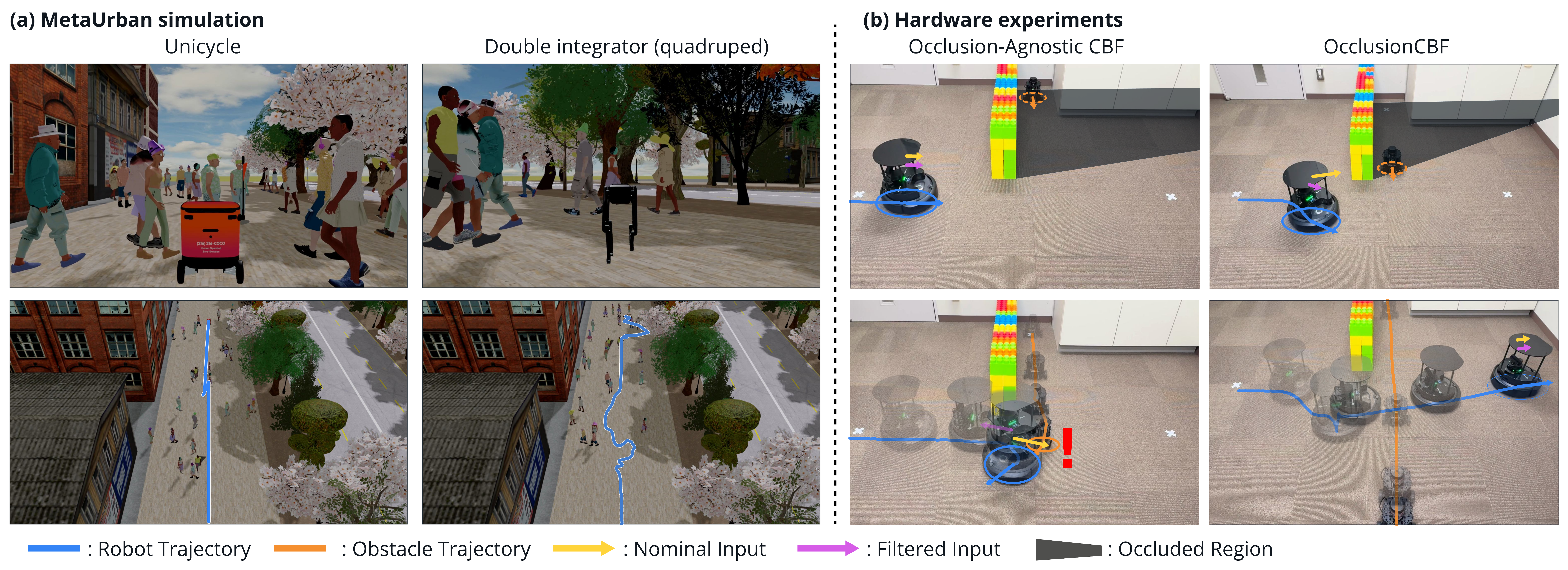}
    \caption{High-fidelity simulation and hardware evaluations. (a) MetaUrban demonstrations for the
    unicycle (left) and double-integrator quadruped (right), with robot-view
    snapshots above and trajectories below. (b) Hardware blind crossing:
    the occlusion-agnostic CBF-QP collides after late detection, whereas
    OcclusionCBF intervenes before detection and reaches the goal.}
    \label{fig:exp}
\vspace{-5pt}
\end{figure*}
 
\textbf{Online computation: }
As reported in Table~\ref{tab:computation_time}, OCBF-QP requires
$4.1$--$4.5\,\textup{ms}$ per update for the unicycle and
$1.6$--$3.9\,\textup{ms}$ for the double integrator. The reactive
CBF-QP is faster because it performs neither occupancy propagation nor
backup certification, but OCBF-QP remains substantially faster than
all predictive-planning baselines. This indicates that occlusion-aware
backup certification can be executed at control rate without online
branched trajectory optimization.

\begin{table}[t]
\centering
\caption{Average controller computation time per update
($\textup{ms}$).}
\label{tab:computation_time}
\footnotesize
\setlength{\tabcolsep}{1pt}
\begin{tabular}{l
  S[table-format=3.1] S[table-format=3.1] S[table-format=3.1]
  S[table-format=3.1] S[table-format=3.1] S[table-format=3.1]}
\toprule
\multirow{2}{*}{Method}
& \multicolumn{3}{c}{Unicycle ($N_{\textup{obs}}$)}
& \multicolumn{3}{c}{Double integrator ($N_{\textup{obs}}$)} \\
\cmidrule(lr){2-4}
\cmidrule(lr){5-7}
& {10} & {20} & {30} & {10} & {30} & {50} \\
\midrule
Control-Tree MPC~\cite{phiquepal_controltree_2021}
    & 48.3 & 54.2 & 58.9 & 106.7 & 110.7 & 117.0 \\
Single-Hypothesis MPC
    &  7.3 &  8.1 &  8.7 &  13.0 &  13.8 &  15.4 \\
OA-MPC~\cite{firoozi_oampc_2025}
    & 69.2 & 101.5 & 117.5 & 93.8 & 151.4 & 186.9 \\
OACP~\cite{zheng_occlusionaware_2026}
    & 80.5 & 76.9 & 78.5 & 87.4 & 89.4 & 88.8 \\
CBF-QP~\cite{ames_control_2019}
    &  0.7 &  0.7 &  0.8 &  0.6 &  0.7 &  0.8 \\
OCBF-QP
    &  4.1 &  4.4 &  4.5 &  1.6 &  2.3 &  3.9 \\
OCBF-QP ($h_j^S$ relaxed)
    & {--} & {--} & {--} & 1.6 & 2.4 & 3.3 \\
\bottomrule
\end{tabular}
\end{table}

\subsection{High-Fidelity Simulation and Hardware Experiments}

We further demonstrate closed-loop execution in the MetaUrban simulator using a unicycle and a quadruped modeled as a double integrator, as shown in
\autoref{fig:exp}(a). Additional demonstrations and failure cases of the
compared baselines are provided in the supplementary video and project
page.

Motivated by a real-world blind-crossing accident in which a bus occluded
an approaching vehicle from a pedestrian's view, we reproduce the scenario
illustrated in \autoref{fig:ocbf_intro} using wheeled mobile robots and
compare OCBF-QP with the occlusion-agnostic CBF-QP. The hardware setup uses ROS~2, with a TurtleBot~4 as the ego robot and a TurtleBot~3 as the dynamic obstacle, and a wall creating the occluded region. As shown in \autoref{fig:exp}(b), the baseline follows
the nominal input while the obstacle is hidden and reacts too late to avoid
collision, whereas OCBF-QP acts before detection, steers away from the
occluded region, and safely reaches the goal.

\section{CONCLUSION}
\label{sec:conclusion}

This paper presented OcclusionCBF, a planner-agnostic safety filter for navigation with potentially hidden dynamic obstacles. By propagating occluded regions into time-indexed reachable-occupancy sets and certifying a backup rollout to a verified terminal set, the method yields affine constraints for a minimally invasive QP while accounting for robot dynamics and input limits. Under the stated assumptions, we established pointwise QP feasibility, recursive feasibility, and collision avoidance for all hidden-obstacle motions covered by the occupancy prediction. Randomized benchmarks demonstrated the highest success rate among the hard-constrained methods with low-millisecond computation, while MetaUrban and blind-crossing hardware experiments showed proactive intervention before detection and safe goal completion.